\documentclass[onecolumn, 11pt, letterpaper]{article}

\usepackage{algorithm}
\usepackage{algorithmic}
\usepackage{amsmath}\allowdisplaybreaks
\usepackage{amssymb}
\usepackage{amsthm}
\usepackage{color,xcolor}
\usepackage[margin = 1in]{geometry}
\usepackage{hyperref}
\usepackage{ifthen}
\usepackage{mathtools}
\usepackage{braket}

\usepackage{graphicx}
\usepackage{subfigure}
\usepackage{caption}

\hypersetup{colorlinks=true, linkcolor=blue, citecolor=blue, anchorcolor=blue, urlcolor=blue}  

\DeclareMathOperator*{\argmax}{argmax}
\DeclareMathOperator*{\argmin}{argmin}
\DeclareMathOperator*{\tr}{tr} 
\DeclareMathOperator*{\diag}{diag} 
\DeclareMathOperator*{\poly}{poly}

\newtheorem{theorem}{Theorem}
\newtheorem{lemma}{Lemma}

\newtheorem{remark}{Remark}
\newtheorem{corollary}{Corollary}

\newcommand{\eq}[1]{(\ref{eq:#1})}

\newcommand{\sect}[1]{\hyperref[sect:#1]{Section~\ref*{sect:#1}}}
\newcommand{\append}[1]{\hyperref[append:#1]{Appendix~\ref*{append:#1}}}
\newcommand{\lem}[1]{\hyperref[lem:#1]{Lemma~\ref*{lem:#1}}}
\newcommand{\thm}[1]{\hyperref[thm:#1]{Theorem~\ref*{thm:#1}}}
\newcommand{\algo}[1]{\hyperref[algo:#1]{Algorithm~\ref*{algo:#1}}}
\newcommand{\cor}[1]{\hyperref[cor:#1]{Corollary~\ref*{cor:#1}}}
\newcommand{\figg}[1]{\hyperref[fig:#1]{Figure~\ref*{fig:#1}}}
\newcommand{\tab}[1]{\hyperref[tab:#1]{Table~\ref*{tab:#1}}}

\def\>{\rangle}
\def\<{\langle}

\newcommand{\E}{\mathbb{E}} 
\newcommand{\trans}{\top} 

\newcommand{\compilehidecomments}{false}

\ifthenelse{ \equal{\compilehidecomments}{true} }{%
	\newcommand{\tongyang}[1]{}
	\newcommand{\zhijie}[1]{}
	\newcommand{\jialin}[1]{}
	\newcommand{\zongqi}[1]{}
}{
	\newcommand{\tongyang}[1]{{\color{blue!50!black}  [\text{Tongyang:} #1]}}
	\newcommand{\zhijie}[1]{{\color{red!60!black} [\text{Zhijie:} #1]}}
	\newcommand{\jialin}[1]{{\color{brown!60!black} [\text{Jialin:} #1]}}
	\newcommand{\zongqi}[1]{{\color{green!60!black} [\text{Zongqi:} #1]}}
}

\title{Quantum Multi-Armed Bandits and Stochastic Linear Bandits Enjoy Logarithmic Regrets}
\author{Zongqi Wan$^{1,2}$, Zhijie Zhang$^{1,2}$, Tongyang Li$^{3,4}$, Jialin Zhang$^{1,2}$, Xiaoming Sun$^{1,2}$\\\\ 
$^1$ Institute of Computing Technology, Chinese Academy of Sciences\\ 
$^2$ University of Chinese Academy of Sciences\\
$^3$ Center on Frontiers of Computing Studies, Peking University\\
$^4$ School of Computer Science, Peking University\\\\
\{wanzongqi20s, zhangzhijie, zhangjialin, sunxiaoming\}@ict.ac.cn \\ tongyangli@pku.edu.cn
}
\date{}

\begin{document}

\maketitle

\begin{abstract}
Multi-arm bandit (MAB) and stochastic linear bandit (SLB) are important models in reinforcement learning, and it is well-known that classical algorithms for bandits with time horizon $T$ suffer $\Omega(\sqrt{T})$ regret. In this paper, we study MAB and SLB with quantum reward oracles and propose quantum algorithms for both models with $O(\poly(\log T))$ regrets, exponentially improving the dependence in terms of $T$. To the best of our knowledge, this is the first provable quantum speedup for regrets of bandit problems and in general exploitation in reinforcement learning. Compared to previous literature on quantum exploration algorithms for MAB and reinforcement learning, our quantum input model is simpler and only assumes quantum oracles for each individual arm.
 \end{abstract}


\section{Introduction}
Bandits are a fundamental model in reinforcement learning applied to problems where an agent has a fixed set of choices and the goal is to maximize its gain, while each choice's properties are only partially known at the time of allocation but may become better understood as iterations continue~\cite{lattimore2020bandit,sutton2018reinforcement}. Bandits exemplify the exploration-exploitation tradeoff where exploration aims to find the best choice and exploitation aims to obtain as many rewards as possible. Bandits have wide applications in machine learning, operations research, engineering, and many other areas~\cite{chapelle2014simple,lei2017actor,silver2016mastering,villar2015multi}.

In this paper, we investigate two important bandit models: multi-armed bandits and stochastic linear bandits. In the multi-armed bandit (MAB) problem, there are $n$ arms where each arm $i\in[n]\coloneqq\{1,2,\ldots, n\}$ is associated with an unknown reward distribution. We denote the expected reward of arm $i$ as $\mu(i)\in [0,1]$. MAB has $T$ rounds. At round $t=1,2,\ldots,T$, the learner chooses an arm $i_t$ and receives a reward $y_t$, a random variable drawn from the reward distribution of $i_t$.
Denote the best arm with the largest expected reward to be $i^*$.
The goal is to minimize the cumulative regret with respect to the best arm $i^*$ over $T$ rounds:
\begin{align}\label{eq:regret}
R(T)=\sum_{t=1}^{T}\left(\mu(i^*)-\mu(i_t)\right).
\end{align}
In the stochastic linear bandit (SLB) problem, the learner can play actions from a fixed action set $A\subseteq\mathbb{R}^d$. There is an unknown parameter $\theta^*\in\mathbb{R}^d$ which determines the mean reward of each action\footnote{The ``action'' is the same as ``arm'' throughout the paper.}. The expected reward of action $x$ is $\mu(x) = x^{\trans}\theta^*\in [0,1]$. It is often assumed that the action $x$ and the $\theta^{*}$ has bounded $L^2$-norm. That is, for some parameters $L,S>0$,
\begin{align}
	\|x\|_2\leq L \mbox{ for all } x\in A, \mbox{ and } \|\theta^*\|_2\leq S. \label{eq:para bounds}
\end{align}
Let $x^*=\argmax_{x\in A} x^{\trans}\theta^*$ be the action with the largest expected reward. Same as MAB, SLB also has $T$ rounds. In round $t$, the learner chooses an action $x_t\in A$ and observes some realization of the reward $y_t$.
The goal is again to minimize the cumulative regret
\begin{align}\label{eq:regret_lin}
	R(T)=\sum_{t=1}^{T}(x^*-x_t)^{\trans}\theta^*.
\end{align}
Regarding the assumption on the reward distributions for both settings, a common assumption is \emph{$\sigma$-sub-Gaussian}. 
Suppose $y$ is a random reward of some action $x$, then for any $\alpha\in \mathbb{R}$, the noise $\eta = y-\mu(x)$ has zero mean and satisfies
$\E [\exp (\alpha X)]\leq \exp\left(\alpha^2\sigma^2/2\right)$. 
In this paper, we consider the \emph{bounded value} assumption and the \emph{bounded variance} assumption. The bounded value assumption requires all rewards to be in a bounded interval, and without loss of generality, we assume that the reward is in $[0,1]$. The bounded variance assumption requires the variances of all reward distributions to have a universal upper bound $\sigma^2$. Note that the bounded value assumption is more strict than the sub-Gaussian assumption, and the sub-Gaussian assumption is more strict than the bounded variance assumption. Nevertheless, even for the bounded value assumption, regrets of classical algorithms for MAB and SLB suffer from $\Omega(\sqrt{nT})$~\cite{auer2002nonstochastic} and $\Omega(d\sqrt{T})$~\cite{DaniHK08} lower bounds, respectively. Fundamentally different models are required to overcome this $\Omega(\sqrt{T})$ bound and quantum computation comes to our aid. 

In our quantum bandit models, we assume that each time we pull an arm, instead of observing an immediate reward as our feedback, we get a chance to access a quantum unitary oracle $\mathcal{O}$ or its inverse $\mathcal{O}^{\dagger}$ once, where $\mathcal{O}$ encodes the reward distribution of this arm,
\begin{align}\label{eq:oracle intro}
\mathcal{O}\colon |0\rangle\rightarrow \sum_{\omega\in \Omega}\sqrt{P(\omega)}|\omega\rangle|y(\omega)\rangle
\end{align}
where $y\colon\Omega \rightarrow \mathbb{R}$ is a random variable whose domain is a finite sample space, and $P$ is a probability measure of $\Omega$. Such an oracle plays the role of reward in our quantum bandit models, and it is a natural generalization of classical reward models. If we perform a measurement on $\mathcal{O}|0\rangle$ with the standard basis immediately after we invoke $\mathcal{O}$, we will observe a realization of the random reward $y$, reducing to the classical bandit models. When we consider a learning environment simulated by a quantum algorithm, the simulation directly gives the quantum oracle. Such situation arises in many reinforcement learning settings where the learning agents are in an artificial environment including games AI, autopilot, etc.; see for instance~\cite{dunjko2015framework,dunjko2016quantum}.

Quantum computing is an emerging technology, and there is a surging interest in understanding quantum versions of machine learning algorithms (see for instance the surveys~\cite{arunachalam2017guest,biamonte2017quantum,dunjko2018machine,schuld2015introduction}). For bandit problems, Casal{\'e} et al.~\cite{casale2020quantum} initiated the study of quantum algorithms for best-arm identification of MAB, and Wang et al.~\cite{wang2021quantum} proved optimal results for best-arm identification of MAB with Bernoulli arms. As an extension, Wang et al.~\cite{wang2021RL} proposed quantum algorithms for finding an optimal policy for a Markov decision process with quantum speedup. These results focused on exploration of reinforcement learning models, and in terms of the tradeoff between exploration and exploitation, the only work we are aware of is~\cite{lumbreras2021multiarm}, which proved that the regret of online learning of properties of quantum states has lower bounds $\Omega(\sqrt{T})$. As far as we know, the quantum \emph{speedup} of the regret of bandit models is yet to explore.

\paragraph{Contributions.}
We initiate the study of quantum algorithms for bandits, including MAB and SLB. Specifically, we formulate the quantum counterparts of bandit models where the learner can access a quantum oracle which encodes the reward distributions of different arms. For both models, we propose quantum algorithms which achieve $\poly(\log T)$ regret. Our results are summarized in \tab{main}. Note that our regret bounds have a worse dependence on $n$ and $d$ because pulling every arm once already incurs $\Omega(n)$ regret, and as a result, the $\Omega(n)$ factor is inevitable in our current algorithm.
\begin{table}[ht]
\centering
\begin{tabular}{ccccc}
\hline
Model & Reference & Setting & Assumption & Regret \\ \hline\hline
MAB & \cite{audibert2009minimax,auer2002nonstochastic,lattimore2020bandit} & Classical & sub-Gaussian & $\Theta(\sqrt{nT})$ \\ \hline
MAB & \thm{qucb regret} & Quantum & bounded value & $O(n\log (T))$ \\ \hline
MAB & \thm{bv qucb regret} & Quantum & bounded variance & $O(n\log^{5/2}(T)\log\log (T))$ \\
\hline\hline
SLB & \cite{abbasi2011improved,DaniHK08,lattimore2020bandit,rusmevichientong2010linearly} & Classical & sub-Gaussian & $\widetilde{\Theta}\left(d\sqrt{T}\right)$ \\ \hline
SLB & \thm{qlb regret} & Quantum & bounded value & $O(d^2\log^{5/2}(T))$ \\ \hline
SLB & \thm{bv qlb regret} & Quantum & bounded variance & $O(d^2\log^4 (T)\log\log(T))$ \\ 
\hline
\end{tabular}
\caption{Regret bounds on multi-armed bandits (MAB) and stochastic linear bandits (SLB).}
\label{tab:main}
\end{table}

Technically, we adapt classical UCB-type algorithms by employing Quantum Monte Carlo method~\cite{montanaro2015quantum} ($\mbox{QMC}$) to explore the arms. $\mbox{QMC}$ has a quadratic better sample complexity compared with classical methods, and this is the essence of our quantum speedup of the regret bound. However, different from the classical empirical mean estimator, $\mbox{QMC}$ cannot provide any classical information feedback before it measures its quantum state. This means our quantum algorithm cannot observe feedback each round. Thus we divide consecutive time slots into stages, where arm switching and measurement only happen at the end of a stage. Worse yet, a quantum state collapses when it is measured and can not be reused again. Our algorithms select the length of each stage carefully to address above problems introduced by the quantum subroutine. For MAB, our algorithm $\mbox{QUCB}$ doubles the length of the stage whenever we select an arm which has been selected before to get reward estimations with growing accuracy. For SLB, we introduce weighted least square estimators to generalize this doubling schema to the case where arms are linearly dependent. In this setting, the length of each stage is closely related to the variance-covariance matrix of these weighted least square estimators, and the determinant of this matrix will get halved per stage.

Finally, we corroborate our theoretical findings with numerical experiments. The results are consistent with our theoretical results, visually proving the quantum speedup of bandit problems. We also consider the presence of quantum noise, and discuss the effects of different levels of quantum noise.

\section{Preliminaries of Quantum Computation}
\paragraph{Basics.}
A \emph{quantum state} can be seen as a $L^{2}$-normalized column vector $\vec{x}=(x_1,x_2,\ldots,x_m)$ in Hilbert space $\mathbb{C}^m$. Intuitively, $\vec{x}$ is a superposition of $m$ classical states, and $|x_i|^2$ is the probability for having the $i$-th state. In quantum computation, people use the Dirac \emph{ket} notation $|x\rangle$ to denote the quantum state $\vec{x}$, and denote $\vec{x}^{\dagger}$ by the \emph{bra} notation $\langle x|$. Given two quantum states $|x\rangle\in\mathbb{C}^m$ and $|y\rangle\in\mathbb{C}^n$, we denote their tensor product by $|x\rangle |y\rangle:=(x_{1}y_{1},x_{1}y_{2},\ldots,x_{m}y_{n-1},x_{m}y_{n})$.

To observe classical information from a quantum state, one can perform a \emph{quantum measurement} on this quantum state. Usually, a POVM (positive operator-valued measure) is used, which is a set of positive semi-definite Hermitian matrices $\{E_i\}_{i\in \Lambda}$ satisfying $\sum_{i\in\Lambda}E_i = I$, $\Lambda$ here is the index set of the POVM. After applying the POVM on $|x\rangle$, outcome $i$ is observed with probability $\langle x|E_i|x\rangle$. The assumption $\sum_{i\in\Lambda}E_i = I$ guarantees that all probabilities add up to $1$.

A quantum algorithm applies unitary operators to an input quantum state. In many cases, information of the input instance is encoded in a unitary operator $\mathcal{O}$. This unitary is called a \emph{quantum oracle} which can be used multiple times by a quantum algorithm. It is common to study the \emph{quantum query complexity} of a quantum algorithm, i.e., the number of $\mathcal{O}$ used by the quantum algorithm.

\paragraph{Quantum reward oracle.}
We generalize MAB and SLB to their quantum counterparts, where we can exploit the power of quantum algorithms. Our quantum bandit problems basically follow the framework of classical bandit problems. 
There are also $T$ rounds. In every round, the learner must select an action, and the regret is defined as \eq{regret} and \eq{regret_lin}. For stochastic linear bandits, we also admit the bounded norm assumption \eq{para bounds}. 
The main difference is that, in our quantum version, the immediate sample reward is replaced with a chance to access an unitary oracle $\mathcal{O}_x$ or its inverse encoding the reward distribution $P_x$ of the selected arm $x$. 
Formally, let $\Omega_x$ be the sample space of the distribution $P_x$. We assume that there is a \emph{finite} sample space $\Omega$ such that $\Omega_x \subseteq \Omega$ for all $x\in A$ (or for all $x\in [n]$ in the MAB setting). $\mathcal{O}_x$ is defined as follows:
\begin{align}\label{eq:oracle}
\mathcal{O}_x\colon |0\rangle\rightarrow \sum_{\omega\in \Omega_x}\sqrt{P_x(\omega)}|\omega\rangle|y^x(\omega)\rangle
\end{align}
where $y^x\colon\Omega_x \to \mathbb{R}$ is the random reward associated with arm $x$. We say $\mathcal{O}_x$ \emph{encodes} probability measure $P_x$ and random variable $y^x$. 

During a quantum bandit task, the learner maintains a quantum circuit. At round $t=1,2,\ldots, T$, it chooses an arm $x_t$ by invoking either of the unitary oracles $\mathcal{O}_{x_t}$ or $\mathcal{O}^{\dagger}_{x_t}$ at most once. After this, the immediate expected regret $\mu(x^*)-\mu(x)$ is added to the cumulative regret. The learner can choose whether to perform a quantum measurement at any round. During two successive rounds, it can place arbitrary unitaries in the circuit. We call the bandit problem equipped with the quantum reward oracle defined above the \emph{Quantum Multi-armed Bandits} (QMAB) and \emph{Quantum Stochastic Linear Bandits} (QSLB).

\begin{remark}
In previous papers~\cite{casale2020quantum,wang2021quantum} investigating the quantum best arm identification problem, they consider the MAB model with Bernoulli rewards using a stronger coherent query model allowing superposition between different arms. That is,
\begin{align}
O\colon |i\rangle_I|0\rangle_B\rightarrow |i\rangle_I(\sqrt{p_i}|1\rangle_B+\sqrt{1-p_i}|0\rangle_B)
\end{align}
where the state of quantum register $I$ corresponds to arm $i$, and $p_i$ represents the mean rewards of arm $i$. This is stronger than \eq{oracle} because our model only has $O_{i}$ for each arm separately and cannot apply superpositions on $i$. We adopt \eq{oracle} because in regret minimization, one should emphasize the exploration-exploitation tradeoff, and we must bind the action of the learner and the feedback the learner get in a single round. The coherent model is not suitable for our setting since it explores all arms together, while \eq{oracle} makes more direct and fair comparisons to classical bandit models.
\end{remark}

\paragraph{Quantum Monte Carlo method.}
To achieve quantum speedup for QMAB and QSLB, we use the Quantum Monte Carlo method \cite{montanaro2015quantum} stated below to estimate the mean rewards of actions.
\begin{lemma}[Quantum Monte Carlo method~\cite{montanaro2015quantum}]
    \label{lem:QME}
    Assume that $y:\Omega \to \mathbb{R}$ is a random variable with bounded variance, $\Omega$ is equipped with a probability measure $P$, and the quantum unitary oracle $\mathcal{O}$ encodes $P$ and $y$. 
    \begin{itemize}
        \item If $y\in [0,1]$, there is a constant $C_1>1$ and a quantum algorithm $\mbox{QMC}_1(\mathcal{O},\epsilon,\delta)$ which returns an estimate $\hat{y}$ of $\E[y]$ such that $\Pr\left(|\hat{y}-\E[y]|\ge\epsilon\right)\leq\delta$ using at most $\frac{C_1}{\epsilon}\log\frac{1}{\delta}$ queries to $\mathcal{O}$ and $\mathcal{O}^{\dagger}$.
        \item If $y$ has bounded variance, i.e., $\mbox{Var}(y)\leq \sigma^2$, then for $\epsilon<4\sigma$, there is a constant $C_2>1$ and a quantum algorithm $\mbox{QMC}_2(\mathcal{O},\epsilon,\delta)$ which returns an estimate $\hat{y}$ of $\E[y]$ such that $\Pr\left(|\hat{y}-\E[y]|\ge\epsilon\right)\leq\delta$ using at most $\frac{C_2\sigma}{\epsilon}\log_2^{3/2} (\frac{8\sigma}{\epsilon})\log_2 (\log_2 \frac{8\sigma}{\epsilon})\log\frac{1}{\delta}$ queries to $\mathcal{O}$ and $\mathcal{O}^{\dagger}$.
    \end{itemize}
\end{lemma}
Note that \lem{QME} demonstrates a quadratic quantum speedup in $\epsilon$ for estimating $\E[y]$ because classical methods such as the Chernoff bound take $O(1/\epsilon^2 \log(1/\delta))$ samples to estimate $\E [y]$ within $\epsilon$ with probability at least $1-\delta$. This is a key observation utilized by our quantum algorithms.


\section{Quantum Multi-Armed Bandits}\label{sect:mab}

In this section, we present an algorithm called $\mbox{QUCB}$ (\algo{qucb}) for $\mbox{QMAB}$ with $O(n\log T)$ regret.
$\mbox{QUCB}$ adopts the canonical upper confidence bound (UCB) framework, combined with the doubling trick. During its execution, it maintains three quantities for each arm $i$: an estimate $\hat{\mu}(i)$ of $\mu(i)$, a confidence radius $r_i$ such that $\mu (i)\in [\hat{\mu}(i)-r_i,\hat{\mu}(i)+r_i]$ with high probability. Besides, it maintains $N_i$, the number of queries to $\mathcal{O}_i$ used to generate $r_i$, for guiding the algorithm.

Classical $\mbox{UCB}$ algorithms for MAB also maintain a confidence interval during the MAB game, the length of this confidence interval decreases as the number of times the corresponding arm is selected. To be exactly, if an arm is selected for $N$ rounds, then the length of the confidence interval of its reward is $O\left(\frac{1}{\sqrt{N}}\right)$. Since we can obtain a quadratic speedup by using the Quantum Monte Carlo method to estimate the mean reward of an arm, the length of the confidence interval is expected to be improved to $O\left(\frac{1}{N}\right)$, then it will be enough to derive a $O(n\log T)$ regret bound. But the introduce of Quantum Monte Carlo method leads to another problem, that is, before we make a measurement on the quantum state, we cannot observe any feedback. Moreover, if we measure the quantum state, then the state collapses. We use the doubling trick to solve these problems. 

Concretely, we adaptively divide whole $T$ rounds into several stages. In stage $s$, it first chooses arm $i_s$ which has the largest $\mu(i_s)+r_{i_s}$, i.e., the right endpoint of the arm's confidence interval.
Then, $r_{i_s}$ is reduced by half and $N_{i_s}$ is doubled, and the algorithm plays arm $i_s$ for the next $N_{i_s}$ rounds.
During this stage, $\mbox{QMC}$ is invoked with $N_{i_s}$ queries to $\mathcal{O}_{i_{s}}$ to update a new estimation $\hat{\mu}(i_s)$ which has better accuracy. After having done all above, the algorithm enter into the next stage. The algorithm terminates after it plays $T$ rounds.
We show in \thm{qucb regret} and \cor{qucb ex} that \algo{qucb} achieves an $O(n\log T)$ expected cumulative regret.

\begin{algorithm}[t]
	\caption{$\mbox{QUCB}_1(\delta)$}
	\label{algo:qucb}
	\hspace*{\algorithmicindent} \textbf{Parameters:} fail probability $\delta$
	\begin{algorithmic}[1]
	   \FOR{$i=1\to n$}
	   	\STATE $r_i\gets 1$ and $N_i\gets \frac{C_1}{r_i}\log \frac{1}{\delta}$
	        \STATE play arm $i$ for consecutive $N_i$ rounds
	        \STATE run $\mbox{QMC}_1(\mathcal{O}_i, r_i, \delta)$ to get an estimation $\hat{\mu}(i)$ for $\mu(i)$
	   \ENDFOR
		\FOR{each stage $s=1,2,\ldots$ (terminate when we have used $T$ queries to all $\mathcal{O}_{i}$)}
		\STATE Let $i_s\leftarrow\argmax_i \hat{\mu}(i)+r_i$ (if $\argmax$ has multiple choices, pick an arbitrary one) 
		\STATE update $r_{i_s} \gets r_{i_s}/2$ and $N_{i_s} \gets \frac{C_1}{r_{i_s}}\log \frac{1}{\delta}$
		\STATE Play $i_s$ for the next $N_{i_s}$ rounds, update $\hat{\mu}(i_s)$ by runing $\mbox{QMC}_1(\mathcal{O}_{i_s},r_{i_s}, \delta)$
		\ENDFOR
	\end{algorithmic}
\end{algorithm}

\begin{theorem}\label{thm:qucb regret}
	Let $C_1$ be the constant specified in \lem{QME}. Under the bounded value assumption, with probability at least $1-n\delta \log_2 \left( \frac{T}{nC_1\log \frac{1}{\delta}}\right)$, the cumulative regret of $\mbox{QUCB}_1(\delta)$ satisfies $R(T)\leq 8(n-1)C_1\log \frac{1}{\delta}$.
\end{theorem}
\begin{proof}
    At the end of each stage (including the initialization stage described in line $1$-$5$ of \algo{qucb}), by \lem{QME}, $\mbox{QMC}$ have enough queries to output an estimation $\hat{\mu}(i)$ such that
    \begin{align}\label{eq:correct estimation}
        |\hat{\mu}(i)-\mu(i)|\leq r_i
    \end{align}
    holds with probability at least $1-\delta$ for any $i\in [n]$.
	
	For each arm $i$, let $\mathcal{S}_i$ be the set of stages when arm $i$ is played, and denote $|\mathcal{S}_i|=K_i$. Initial stages are not included in $\mathcal{S}_i$. According to \algo{qucb}, each time we find arm $i$ in the argmax in Line $7$ in some stages, $r_i$ is reduced by half, and $N_{i}$ is doubled subsequently. Then we play arm $i$ for consecutive $N_i$ rounds. This means that the number of rounds of each stage in $\mathcal{S}_i$ are $2C_1\log \frac{1}{\delta}$, $4C_1\log \frac{1}{\delta}$, $\ldots$, $2^{K_i}C_1\log \frac{1}{\delta}$. In total, arm $i$ has been played for $\left(2^{K_i+1}-1\right)C_1\log \frac{1}{\delta}$ rounds. Because the total number of rounds is at most $T$, we have
\begin{align}\label{eq:num-iteration}
nC_1\log \frac{1}{\delta}+\sum_{i=1}^{n}\left(2^{K_{i}+1}-1\right)C_1\log \frac{1}{\delta}\leq T,
\end{align}
where the first term of \eq{num-iteration} is the number of rounds in the initialization stage. Because $2^{x}$ is a convex function in $x\in[0,+\infty)$, by Jensen's inequality we have
\begin{align}
\sum_{i=1}^{n}2^{K_{i}+1}\geq n\cdot 2^{1/n\sum_{i=1}^{n}(K_{i}+1)}.
\end{align}
Plugging this into \eq{num-iteration}, we have 
\begin{align}
    \sum_{i=1}^{n}K_{i}\leq n\log_{2}\left(\frac{T}{nC_1\log \frac{1}{\delta}}\right)-n.
\end{align}
Since $\mbox{QMC}$ is called for $n+\sum_{i=1}^n K_i$ times, by the union bound, with probability at least $1-n\delta \log_2 \left( \frac{T}{nC_1\log \frac{1}{\delta}}\right)$, the output estimate of every invocation of $\mbox{QMC}$ satisfies \eq{correct estimation}. We refer to the event as the \emph{good} event and assume that it holds below.
	
	Recall that $i^*$ is the optimal arm and $i_s$ is the arm chosen by the algorithm during stage $s$.	By the argmax in Line $7$ of \algo{qucb}, $\hat{\mu}(i_s)+r_{i_s}\geq \hat{\mu}(i^*)+r_{i^*}$.
	Under the good event, $\mu(i_s)+r_{i_s}\geq\hat{\mu}(i_s)$ and $\hat{\mu}(i^*)+r_{i^*}\geq \mu(i^*)$.
	Therefore,
	$\mu(i_s)+2r_{i_s}\geq\hat{\mu}(i_s)+r_{i_s}\geq \hat{\mu}(i^*)+r_{i^*}\geq \mu(i^*)$, and it follows that 
	\begin{align}
	    \Delta_{i_s}\coloneqq \mu(i^*)-\mu(i_s)\leq 2 r_{i_s}.
	\end{align}
	For each arm $i$, we denote by $R(T;i)$ the contribution of arm $i$ to the cumulative regret over $T$ rounds. By our notation above, arm $i$ is pulled in $K_{i}$ stages and the initialization stage. In initialization stages it is pulled for $C_1\log \frac{1}{\delta}$ times. In each stage of $\mathcal{S}_i$ it is pulled for $N_{i}=2C_1\log \frac{1}{\delta}$, $4C_1\log \frac{1}{\delta}$, $\ldots$, $2^{K_i}C_1\log \frac{1}{\delta}$ times respectively, and the reward gap $\Delta_{i}\leq 2r_{i}$ in the last stage is $2\cdot \frac{1}{2^{K_i-1}}=2^{2-K_i}$. Note that the index of the stage in $\mathcal{S}_i$ does not influence the gap $\Delta_{i_s}$. Therefore, we can use $2^{2-K_{i}}$ to bound the gap of $\mu(i^*)$ and $\mu(i_s)$. For those arm $i$ which are only pulled in the initialization stage, we bound the their reward gap to $1$. Thus, we have
	\begin{align}
	R(T;i)\leq \max\left\{\sum_{k=0}^{K_i}2^{k}C_1\log \frac{1}{\delta} \cdot 2^{2-K_{i}},C_1\log \frac{1}{\delta}\right\}\leq 8C_1\log \frac{1}{\delta}.
	\end{align}
The cumulative regret is the summation of $R(T;i)$ for $i\neq i^*$; we have $R(T)=\sum_{i\neq i^*}R(T;i)\leq 8(n-1)C_1\log \frac{1}{\delta}$. It completes the proof since we have good event with probability at least $1-n\delta \log_2 \left( \frac{T }{nC_1\log \frac{1}{\delta}}\right)$.
\end{proof}

\begin{corollary}\label{cor:qucb ex}
    Set $\delta = \frac{1}{T}$, $\mbox{QUCB}_1(\delta)$ satisfies $\E [R(T)]\leq \left(8(n-1)C_1+1\right)\log_2 T=O(n\log T)$.
\end{corollary}
\begin{proof}
    Let $\mathcal{E}$ denote the good event in \thm{qucb regret}. Then 
    \begin{align}
        \E [R(T)] &\leq \left(1-\frac{n}{T}\log_2 \left(\frac{T}{nC\log \frac{1}{T}}\right)\right) \E\left[R(T)|\mathcal{E}\right]+\frac{n}{T}\log_2 \left(\frac{T}{nC\log \frac{1}{T}}\right) \E\left[R(T)|\overline{\mathcal{E}}\right]\\&\leq 8(n-1)C\log_2 T + \frac{n}{T}\log_2 \left(\frac{T}{nC\log \frac{1}{T}}\right)\cdot T\leq \left(8(n-1)C+1\right)\log_2 T.
    \end{align}
\end{proof}

For the bounded variance assumption, we can slightly modify \algo{qucb} to obtain a new algorithm called $\mbox{QUCB}_2(\delta)$ and bound its regret to poly-logarithmic order with similar proofs. The proofs of \thm{bv qucb regret} below are deferred to \append{sec-3-proof}.
\begin{theorem}\label{thm:bv qucb regret}
Let $C_2$ be the constant in \lem{QME}. Under the bounded variance assumption, with probability at least $1-n\delta \log_2 \left( \frac{T}{nC_2\log \frac{1}{\delta}}\right)$, the cumulative regret of $\mbox{QUCB}_2(\delta)$ satisfies $R(T)\leq O\left(n\sigma\log^{3/2}(T)\log\log (T)\log \frac{1}{\delta}\right)$. Moreover, setting $\delta=1/T$, $\E [R(T)]=O\left(n\sigma\log^{5/2} (T)\log\log (T)\right)$.
\end{theorem}


\section{Quantum Stochastic Linear Bandits}\label{sect:slb}
In this section, we present the algorithm $\mbox{QLinUCB}$ as well as its analysis for quantum stochastic linear bandits, showing a $\poly(\log T )$ regret bound of $\mbox{QLinUCB}$. Recall the obstacles we encountered when we design the algorithm for QMAB, where the doubling trick is used to solve the problem introduced by the quantum subroutine. In the QSLB setting, we face the same problem. However, the doubling trick cannot be used explicitly to solve the problem in the QSLB setting. We aim to obtain a regret bound which have dependence on the dimension of the action set rather than its size. In fact, similar to classical SLB, we allow the action set to be infinite. Thus, we must consider the linear dependence of different arms, and generalize the doubling trick used in \sect{mab} to fit this situation.

As a variant of the classical algorithm $\mbox{LinUCB}$, $\mbox{QLinUCB}$ also adopts the canonical upper confidence bound (UCB) framework. It runs in several stages.
In stage $s$, it first constructs a confidence region $\mathcal{C}_{s-1}$ for the true parameter $\theta^*$, and then picks the best action $x_s\in A$ over $\mathcal{C}_{s-1}$. After $x_s$ is determined, it sets an carefully selected accuracy value $\epsilon_s$ for stage $s$ and plays action $x_s$ for the next $\frac{C_1}{\epsilon_s}\log\frac{m}{\delta}$ rounds, where $m\coloneqq d\log(\frac{L^2T^2}{d\lambda}+1)$ is an upper bound for the number of total stages, see \lem{stages}.
When playing action $x_s$ during this stage, the algorithm implements a quantum circuit for $\mbox{QMC}(\mathcal{O}_{x_s} ,\epsilon_s,\frac{\delta}{m})$ and gets an estimate $y_s$ of $x_s^{\trans}\theta^*$ with accuracy $\epsilon_s$ and error probability less than $\delta/m$.
After that, it updates the estimate $\hat{\theta}_s$ of $\theta^*$ using a weighted least square estimator. That is,
 \begin{align}\label{eq:RWLS}
 	\hat{\theta}_s = \argmin_{\theta\in\Theta} \sum_{k=1}^{s} \frac{1}{\epsilon_k^2}\|x_k^{\trans}\theta-y_k\|_2^2+\lambda\|\theta\|_2^2,
 \end{align}
where $\lambda$ is a regularizing parameter. We give estimates $y_k$ different weights according to their accuracy in this least square estimator. The estimator \eq{RWLS} has simple closed-form solution as follows. Let $V_s=\lambda I + \sum_{k=1}^s \frac{1}{\epsilon_k^2}x_kx_k^{\trans}=\lambda I + X_s^{\trans}W_sX_s\in\mathbb{R}^{d\times d}$. Then, $\hat{\theta}_{s}:=V_s^{-1}X_{s}^{\trans}W_{s}Y_{s}$, where $X_s,Y_s,W_s$ are defined in line $9$ of \algo{quantum linear bandits}. Besides, with the definition of $V_s$, $\mbox{QLinUCB}$ actually sets $\epsilon_s = \|x_s\|_{V_{s-1}^{-1}}$ where $V_{s-1}$ is calculated in stage $s-1$. Our choice of $\epsilon_s$ and the $\frac{1}{\epsilon_k^2}$ weight of the least square estimator in \eq{RWLS} are the key components of the quantum speedup of $\mbox{QLinUCB}$.
\begin{algorithm}[t]
\caption{$\mbox{QLinUCB}_1(\delta)$}
\label{algo:quantum linear bandits}
	\hspace*{\algorithmicindent} \textbf{Parameters:} fail probability $\delta$
\begin{algorithmic}[1]
    \STATE Initialize $V_0\gets \lambda I_d$, $\hat{\theta}_0\gets \mathbf{0}\in \mathbb{R}^d$ and $m\gets d\log(\frac{L^2T^2}{d\lambda}+1)$.
	\FOR{each stage $s=1,2,\ldots$ (terminate when we have used $T$ queries to all $\mathcal{O}_{i}$)}
	    \STATE $\mathcal{C}_{s-1}\gets\{\theta\in\mathbb{R}^d: \|\theta-\hat{\theta}_{s-1}\|_{V_{s-1}}\leq \lambda^{1/2}S+\sqrt{d(s-1)}\}$.
	    \STATE $(x_s,\tilde{\theta}_s)\gets\argmax_{(x,\theta)\in A\times\mathcal{C}_{s-1}} x^{\trans}\theta$.
	    \STATE $\epsilon_{s} \gets \|x_s\|_{V_{s-1}^{-1}}$.
	    \FOR{the next $\frac{C_1}{\epsilon_s}\log\frac{m}{\delta}$ rounds}
	    \STATE Play action $x_s$ and run $\mbox{QMC}_1(\mathcal{O}_{x_s},\epsilon_s,\delta/m)$, getting $y_{s}$ as an estimation of $x_s^{\trans}\theta^*$.
	    \ENDFOR
	    \STATE $X_s\gets (x_1,x_2,\ldots, x_s)^{\trans}\in\mathbb{R}^{s\times d}$, $Y_s\gets(y_1,y_2,\ldots,y_s)^{\trans}\in \mathbb{R}^{s}$ and  $W_s \gets \diag\left(\frac{1}{\epsilon_1^2},\frac{1}{\epsilon_2^2},\ldots, \frac{1}{\epsilon_s^2}\right)$
	    \STATE Update $V_s\gets V_{s-1}+\frac{1}{\epsilon_{s}^2}x_sx_s^{\trans}$ and $\hat{\theta}_{s}\gets V_s^{-1}X_{s}^{\trans}W_{s}Y_{s}$.
	\ENDFOR
\end{algorithmic}
\end{algorithm}

If we investigate $\det(V_k)$ for stage $k$, we will find it get doubled each stage, that is, $\det(V_{k+1})=2\det(V_k)$. That is why we say our algorithm for QSLB uses an implicit doubling trick. With this doubling $\det(V_k)$, we can control the number of stages under $O(d\log T)$.

\begin{lemma}
\label{lem:stages}
\algo{quantum linear bandits} has at most $m=d\log(\frac{L^2T^2}{d\lambda}+1)$ stages, where $\lambda$ is the regularizing parameter in \eq{RWLS}.
\end{lemma}
\begin{proof}
We show that if \algo{quantum linear bandits} executes $m$ stages, then at least $T$ rounds are played, which proves the lemma.
We first give a lower bound for $\sum_{k=1}^m \frac{1}{\epsilon_k^2}$.
For $k\geq 0$,
\begin{align}
    \det(V_{k+1}) &= \det\left(V_{k}+\frac{1}{\epsilon_{k+1}^{2}} x_{k+1}x_{k+1}^{\trans}\right) \\
    &=\det\left(V_{k}^{1/2}\left(I+\frac{1}{\epsilon_{k+1}^{2}}V_{k}^{-1/2}x_{k+1}x_{k+1}^{\trans}V_{k}^{-1/2}\right)V_{k}^{1/2}\right) \\
    &=\det(V_{k})\det\left(I+\frac{1}{\epsilon_{k+1}^{2}}V_{k}^{-1/2}x_{k+1}x_{k+1}^{\trans}V_{k}^{-1/2}\right) \\
    &=\det(V_{k})\left(1+\left\|\frac{1}{\epsilon_{k+1}}V_k^{-1/2}x_{k+1}\right\|^2\right) \\
    &=\det(V_{k})\left(1+\frac{1}{\epsilon_{k+1}^2}\|x_{k+1}\|_{V_k^{-1}}^2\right) \\
    &=2\det(V_k).
\end{align}
Thus, $\det(V_{m})=2^{m}\det(V_{0})=2^m \lambda^d$.
On the other hand,
\begin{align}
    \tr(V_m)=d\lambda + \sum_{k=1}^m \frac{\|x_k\|^2}{\epsilon_k^2}\leq d\lambda + \sum_{k=1}^m \frac{L^2}{\epsilon_k^2}.
\end{align}
By the trace-determinant inequality,
\begin{align}
    d\lambda + \sum_{k=1}^m \frac{L^2}{\epsilon_k^2}\geq\tr(V_m)\geq d\cdot\det(V_{m})^{1/d}=d\lambda\cdot 2^{m/d}.
\end{align}
Hence,
\begin{align}
    \sum_{k=1}^m \frac{1}{\epsilon_k^2}\geq\frac{d\lambda}{L^2}(2^{m/d}-1).
\end{align}
Since the $k$-th stage contains $\frac{C_1}{\epsilon_k}\log\frac{m}{\delta}$ rounds, the first $m$ stages contain $\sum_{k=1}^{m}\frac{C_1}{\epsilon_{k}}\log\frac{m}{\delta}$ rounds in total.
By the above argument and the value of $m$, we have
\begin{align}
    \sum_{k=1}^{m}\frac{C}{\epsilon_{k}}\log\frac{m}{\delta}\geq \sum_{k=1}^{m}\frac{1}{\epsilon_{k}} \geq \sqrt{\sum_{k=1}^{m}\frac{1}{\epsilon_{k}^{2}}}\geq  \frac{1}{L}\sqrt{d\lambda (2^{m/d}-1) }\geq T.
\end{align}
This means that \algo{quantum linear bandits} has at most $m$ stages.
\end{proof}

Then we show in the following lemma that the confidence regions we construct in each stage contain the true parameter $\theta^*$ with high probability.
\begin{lemma}
\label{lem:trust region}
With probability at least $1-\delta$, for all $s\geq 0$, $\theta^*\in\mathcal{C}_s\coloneqq\{\theta\in\mathbb{R}^d: \|\theta-\hat{\theta}_s\|_{V_s}\leq \lambda^{1/2}S+\sqrt{ds} \}$.
\end{lemma}

\begin{proof}
First note that in stage $s$, $|x_s^{\trans}\theta^*-y_s|\leq \epsilon_s$ with probability at least $1-\delta/m$.
By \lem{stages}, there are at most $m$ stages.
Thus, by union bound, with probability at least $1-\delta$, for all $1\leq s\leq m$,  $|x_s^{\trans}\theta^*-y_s|\leq \epsilon_s$.
We assume this holds in the remaining part of the proof.

Write $\theta^*$ as
\begin{align}
    \theta^*=V_{s}^{-1}V_{s}\theta^*=V_{s}^{-1}(\lambda I+X_{s}^{\trans}W_{s}X_{s})\theta^*=\lambda V_{s}^{-1}\theta^*+V_{s}^{-1}X_{s}^{\trans}W_{s}X_{s}\theta^*.
\end{align}
Then, for any direction $x\in\mathbb{R}^d$, by the Cauchy-Schwarz inequality,
\begin{align}
    |x^{\trans}(\theta^*-\hat{\theta}_s)| &=|x^{\trans}(\lambda V_{s}^{-1}\theta^*+V_{s}^{-1}X_{s}^{\trans}W_{s}(X_{s}\theta^*-Y_s))| \\
    &\leq \|x\|_{V_s^{-1}}\left(\lambda \|V_s^{-1}\theta^*\|_{V_s}+\|V_s^{-1}X^{\trans}W_s(X_s\theta^*-Y_s)\|_{V_s}\right)\\
    &\leq \|x\|_{V_s^{-1}}\left(\lambda\|\theta^*\|_{V_s^{-1}}+\|X_{s}^{\trans}W_{s}(X_{s}\theta^*-Y_s)\|_{V_s^{-1}}\right).
\end{align}
Plugging in $x=V_s(\theta^*-\hat{\theta}_s)$ and rearranging the above inequality, we get
\begin{align}\label{eq:coarse radius}
    \|\theta^*-\hat{\theta}_s\|_{V_s}\leq \lambda\|\theta^*\|_{V_s^{-1}}+\|X_{s}^{\trans}W_{s}(X_{s}\theta^*-Y_s)\|_{V_s^{-1}}.
\end{align}

Next, we bound the second term on the right-hand side of the inequality \eq{coarse radius}.
Define $\Gamma_{s}\coloneqq W_{s}^{1/2}(X_{s}\theta^*-Y_{s})\in\mathbb{R}^s$.
Then,
\begin{align}
    \|X_{s}^{\trans}W_{s}(X_{s}\theta^*-Y_s)\|_{V_s^{-1}}^2=\|X_{s}^{\trans}W_{s}^{1/2}\Gamma_s\|_{V_s^{-1}}^2=\Gamma_s^{\trans}W_s^{1/2}X_sV_s^{-1}X_{s}^{\trans}W_{s}^{1/2}\Gamma_s.
\end{align}
Note that $W_{s}^{1/2}=\diag(1/\epsilon_1,1/\epsilon_2,\cdots,1/\epsilon_s)$, and the $k$-th component of $X_s\theta^*-Y_s$ is $x_k^{\trans}\theta^*-y_k$, which satisfies $|x_k^{\trans}\theta^*-y_k|\leq\epsilon_k$.
Therefore, $\|\Gamma_{s}\|_2\leq \sqrt{s}\cdot\|\Gamma_{s}\|_{\infty}\leq \sqrt{s}$.
On the other hand, since $V_{s}$ is positive definite, we point out that $W_s^{1/2}X_sV_s^{-1}X_{s}^{\trans}W_{s}^{1/2}$ is an $s\times s$ positive semi-definite matrix.
Hence, all of its eigenvalues are non-negative, and therefore its maximum eigenvalue is no larger than its trace, the sum of its eigenvalues.
That is,
\begin{align}
    \|W_{s}^{1/2}X_{s}V_{s}^{-1}X_{s}^{\trans}W_{s}^{1/2}\|_2 \leq &\tr \left(W_{s}^{1/2}X_{s}V_{s}^{-1}X_{s}^{\trans}W_{s}^{1/2}\right)\\ = & \tr \left(V_{s}^{-1}X_{s}^{\trans}W_{s}X_{s}\right) \\ =& \tr \left(I_{d}-\lambda V_s^{-1}\right) \leq d.
\end{align}
Therefore,
\begin{align}\label{eq:ds bound}
    \|X_{s}^{\trans}W_{s}(X_{s}\theta^*-Y_s)\|_{V_s^{-1}}^2\leq \|\Gamma_{s}\|_2^2\cdot \|W_{s}^{1/2}X_{s}V_{s}^{-1}X_{s}^{\trans}W_{s}^{1/2}\|_2 \leq ds.
\end{align}
Finally, note that $\|\theta^*\|_{V_s^{-1}}\leq \|\theta^*\|_2/\sqrt{\lambda}\leq S/\sqrt{\lambda}$, we have
\begin{align}
    \|\theta^*-\hat{\theta}_s\|_{V_s}\leq \lambda^{1/2}S+\sqrt{ds}.
\end{align}
\end{proof}

Together with \lem{stages} and \lem{trust region}, following the standard optimism in the face of uncertainty proof we can bound the cumulative regret of each stage to $O\left(d\sqrt{\log T}\right)$, leading to our regret bound.

\begin{theorem}
\label{thm:qlb regret}
Under the bounded value assumption, with probability at least $1-\delta$, the regret of $\mbox{QLinUCB}_1(\delta)$ satisfies \begin{align}
    R(T)=O\left(d^2\log^{3/2}\left(\frac{L^2T^2}{d\lambda}+1\right)\log \frac{d\log(\frac{L^2T^2}{d\lambda}+1)}{\delta}\right).
\end{align} Moreover, the expected regret of $\mbox{QLinUCB}_2(\frac{m}{T})$ satisfies $\E [R(T)] = O\left(d^2\log^{5/2}\left(\frac{L^2T^2}{d\lambda}+1\right)\right)$.
\end{theorem}

\begin{proof}
In stage $s$, the algorithm plays action $x_s$ for $\frac{C_1}{\epsilon_s}\log \frac{m}{\delta}$ rounds.
The regret in each round is $(x^{*}-x_s)^{\trans}\theta^*$.
By the choice of $(x_s,\tilde{\theta}_s)$,
\begin{align}
    (x^{*})^{\trans}\theta^* \leq x_s^{\trans}\tilde{\theta}_s.
\end{align}
Therefore, by Cauchy-Schwarz inequality,
\begin{align}
    (x^{*}-x_s)^{\trans}\theta^* &\leq x_s^{\trans}(\tilde{\theta}_s-\theta^*)\leq\|x_s\|_{V_{s-1}^{-1}}\|\tilde{\theta}_s-\theta^*\|_{V_{s-1}} \\
    &\leq \|x_s\|_{V_{s-1}^{-1}}(\|\tilde{\theta}_s-\hat{\theta}_{s-1}\|_{V_{s-1}}+\|\hat{\theta}_{s-1}-\theta^*\|_{V_{s-1}}) \\
    &= \epsilon_s\cdot(\|\tilde{\theta}_s-\hat{\theta}_{s-1}\|_{V_{s-1}}+\|\hat{\theta}_{s-1}-\theta^*\|_{V_{s-1}}).
\end{align}
By the choice of $\tilde{\theta}_s$ and \lem{trust region}, with probability at least $1-\delta$, for all stages $s\geq 1$, both $\tilde{\theta}_{s}$ and $\theta^*$ lie in $\mathcal{C}_{s-1}$. Thus,
\begin{align}
    (x^{*}-x_s)^{\trans}\theta^* \leq 2\epsilon_s\cdot (\lambda^{1/2}S+\sqrt{d(s-1)}).
\end{align}
The cumulative regret in stage $s$ is therefore bounded by
\begin{align}
    2C_1(\lambda^{1/2}S+\sqrt{d(s-1)})\log\frac{m}{\delta}.
\end{align}
Since there are at most $m$ stages by \lem{stages}, the cumulative regret over all stages and rounds satisfies
\begin{align}
R(T) & \leq \sum_{s=1}^{m} 2C_1(\lambda^{1/2}S+\sqrt{d(s-1)})\log\frac{m}{\delta} \\
 &= O\left(\lambda^{1/2}Sm+d^{1/2}m^{3/2}\log\frac{m}{\delta}\right)\\
  &= O\left(d^2\log^{3/2}\left(\frac{L^2T^2}{d\lambda}+1\right)\log \frac{d\log(\frac{L^2T^2}{d\lambda}+1)}{\delta}\right).
\end{align}

For expected regret bound, let $\mathcal{E}$ be the event that the bound in \thm{qlb regret} holds.
	Note that for any $x_1,x_2\in A$,
	\begin{align}
		|(x_1-x_2)^{\trans}\theta^*|\leq \|x_1-x_2\|_2\|\theta^*\|_2\leq (\|x_1\|_2+\|x_2\|)\|\theta^*\|_2\leq 2LS.
	\end{align}
    Then, we have
    \begin{align}
    	\E[R(T)] &=\E[R(T)\mid \mathcal{E}]\Pr[\mathcal{E}]+\E[R(T)\mid \bar{\mathcal{E}}]\Pr[\bar{\mathcal{E}}] \\
    	&\leq O\left(\lambda^{1/2}Sm+d^{1/2}m^{3/2}\log(T)\right)+2LST\frac{m}{T} \\
    	&=O\left(d^{2}\log^{3/2}\left(\frac{L^2T^2}{d\lambda}+1\right)\log T\right).
    \end{align}

\end{proof}

For the bounded variance assumption, we have a similar result with an additional overhead of $O\left(\log^{3/2}(T)\log\log(T)\right)$ in the regret bound.
\begin{theorem}
\label{thm:bv qlb regret}
Under the bounded variance assumption, with probability at least $1-\delta$, the regret of $\mbox{QLinUCB}_2(\delta)$ satisfies $R(T)=O\left(\sigma d^2\log^3\left(\sigma LT\right)\log\log\left(\sigma LT\right)\log \frac{m}{\delta}\right)$. Moreover, setting $\delta = \frac{m}{T}$, we have $\E [R(T)] = O\left(\sigma d^2\log^4\left(\sigma LT\right)\log\log\left(\sigma LT\right)\right)$.
\end{theorem}

The proof details of the \thm{bv qlb regret} are deferred to \append{sec-4-proof}.


\section{Numerical Experiments}\label{sect:numerics}
We conduct experiments to demonstrate the performance of our two quantum variants of bandit algorithms. For simplicity, we use the Bernoulli rewards in both bandit settings. When considering the Bernoulli noise, we can use the Quantum Amplitude Estimation algorithm in \cite{brassard2002amplitude} as our mean estimator. In \sect{no noise}, we perform the simulations without the quantum noise. In this case, we can run algorithms for a huge amount of rounds to show the advantage of $\mbox{QUCB}$ and $\mbox{QLinUCB}$ on regret. In \sect{noise}, we consider the presence of quantum noise and study the influence of quantum noise to regret. Specifically, we consider a widely used quantum noise model called depolarizing noise. For all experiments, we repeat for $100$ times and calculate the average regret and standard deviation. Our experiments are executed on a computer equipped with Xeon E5-2620 CPU and 64GB memory. Additional details of numerical experiments are given in \append{more-numerics}.

\subsection{Experiments without quantum noise}
\label{sect:no noise}
\paragraph{QMAB setting.} For the QMAB setting, we run $\mbox{UCB}$\footnote{We use the version described in \cite{lattimore2020bandit}, chapter 7, Theorem 7.1.} and $\mbox{QUCB}$ on a 2-arm bandit for $T=10^6$ rounds. Classical $\mbox{UCB}$ algorithm has an instance-dependent $O(\log T)$ bound. Let $SA$ be the set of all sub-optimal arm, then its cumulative regret $R(T)$ satisfies
\begin{align}
    \E [R(T)] \leq \sum_{i\in SA} \frac{1}{\Delta_i} \log T
\end{align}
by~\cite{lattimore2020bandit}, where $\Delta_i$ is the reward gap of arm $i$. That is to say, if the reward gap is independent with time horizon $T$, then $\mbox{UCB}$ also has an $O(\log T)$ regret. Thus, to compare $\mbox{UCB}$ and $\mbox{QUCB}$, we set the reward gap of our experimental instance relatively small. Overall, we set the mean reward of the optimal arm to be $0.5$, then we try different reward gap of the sub-optimal arm, including $0.01, 0.005$, and $0.002$. The results are shown in \figg{MAB_experiment}. From the result, we can see that $\mbox{QUCB}$ has much lower expected regret and variance than $\mbox{UCB}$ when the reward gap is small. In (c), since the reward gap is small enough, $\mbox{UCB}$ cannot distinguish these two arms within $10^6$ rounds thus it suffers almost linear regret, but our $\mbox{QUCB}$ successfully distinguishes them. Furthermore, we find that even if we set the parameter $\delta$ much greater than $1/T$ which are the parameter used in \cor{qucb ex}, the regret still maintain low variance, which means the error probability is not as high as our theoretical bound.

  \begin{figure}[ht]
      \centering
      \subfigure[Reward gap $=0.01$]{
          \includegraphics[width=0.31\textwidth]{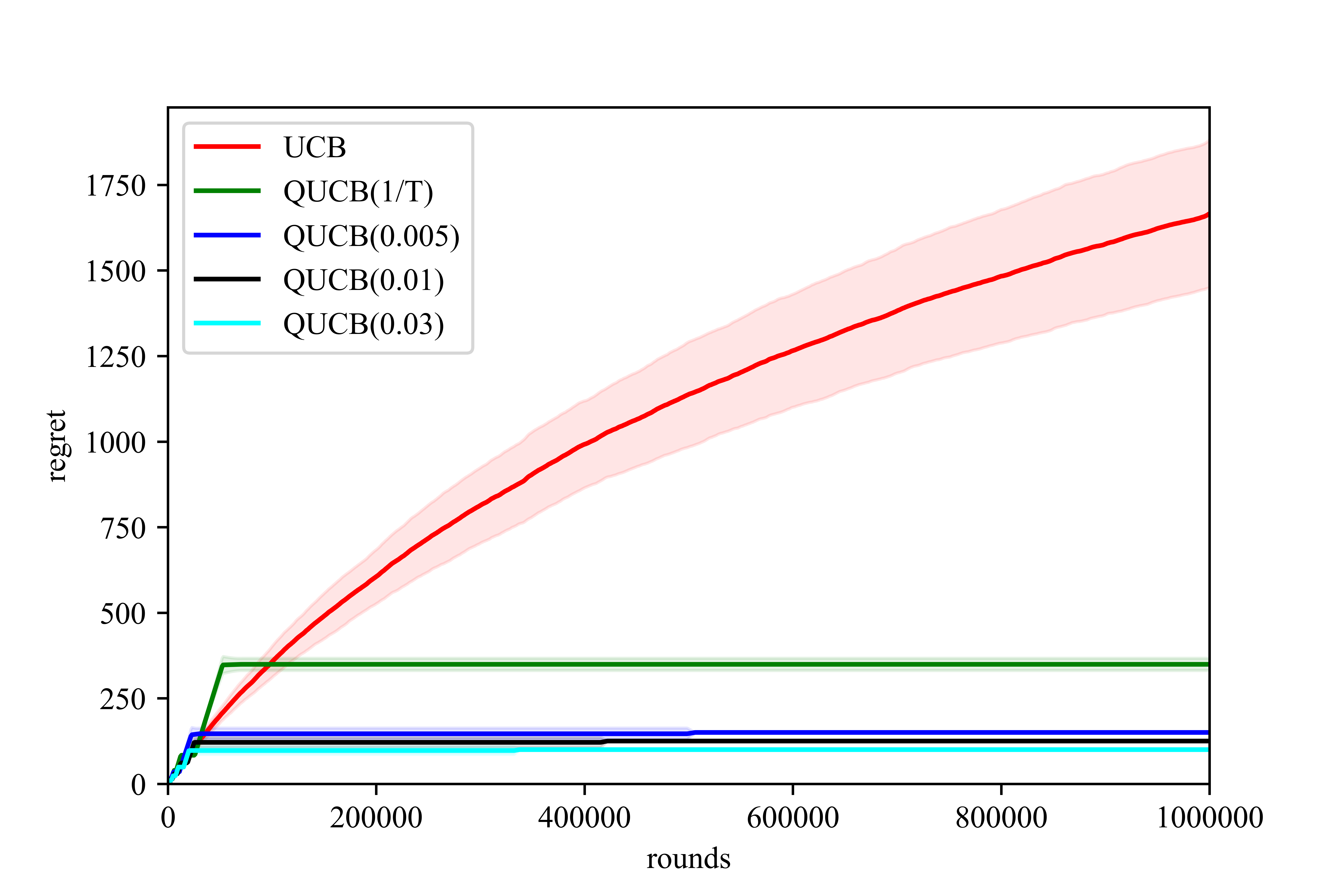}
      }
      \subfigure[Reward gap $=0.005$]{
          \includegraphics[width=0.31\textwidth]{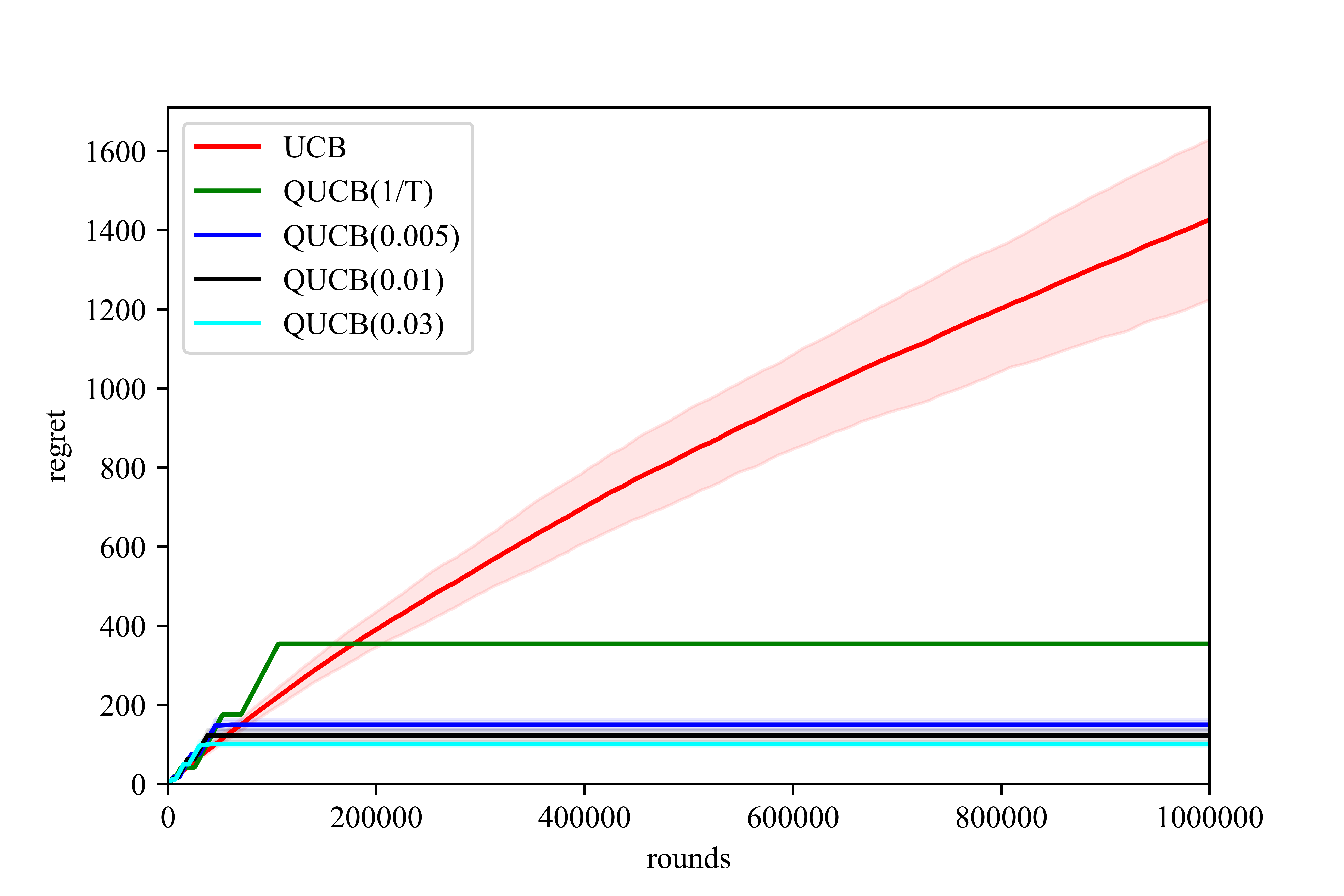}
      }
      \subfigure[Reward gap $=0.002$]{
          \includegraphics[width=0.31\textwidth]{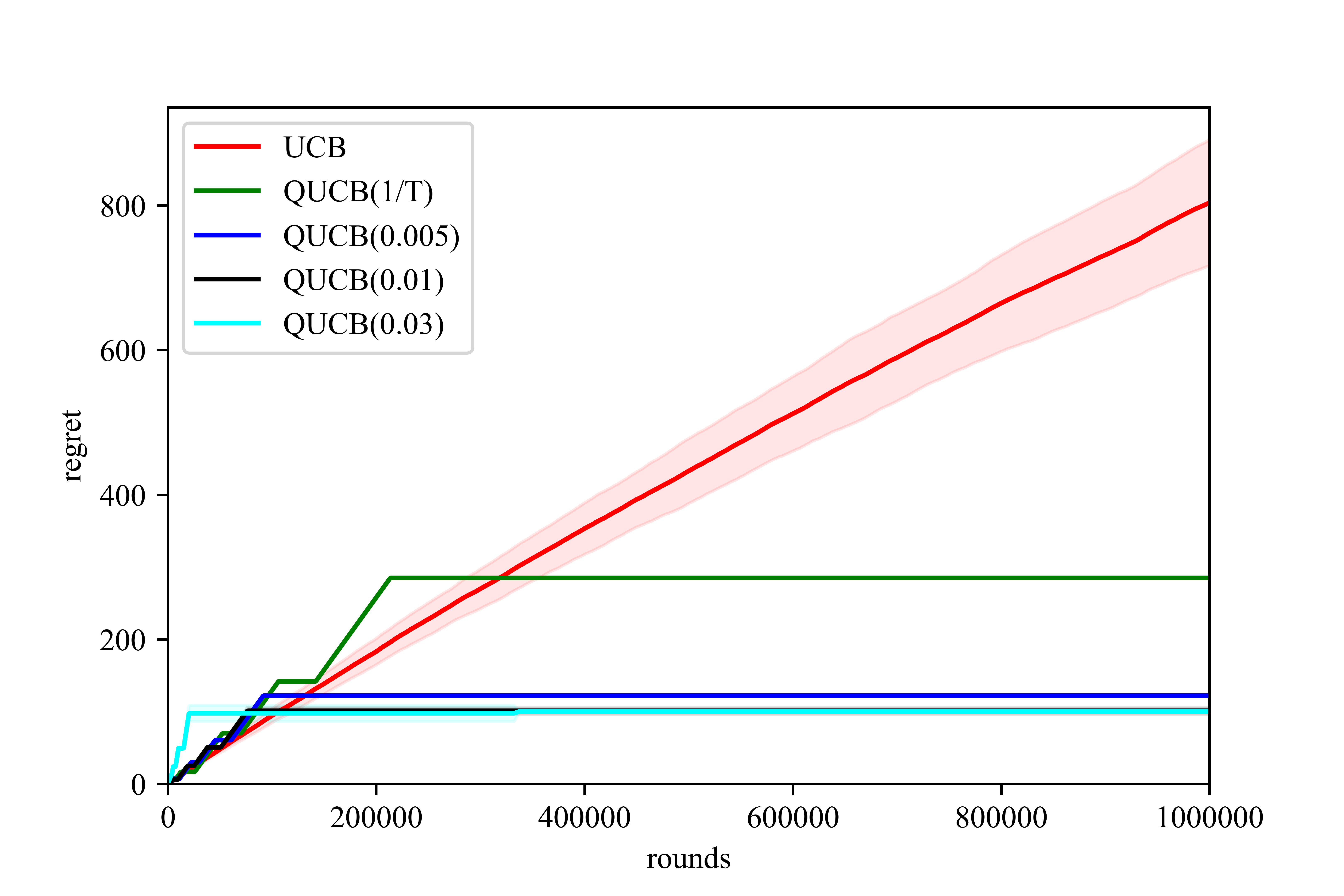}
      }
      \caption{Comparison between $\mbox{UCB}$ and $\mbox{QUCB}$, $\mbox{QUCB}(0.01)$ means $\delta=0.01$.}
      \label{fig:MAB_experiment}
  \end{figure}

\paragraph{QSLB setting.} For the QSLB setting, we study an instance in $\mathbb{R}^{2}$. We take time horizon $T=10^6$. We use the finite action set, and spread $50$ actions equally spaced on the positive part of the unit circle. We set parameter $\theta^*=(\cos (0.35\pi),\sin (0.35\pi))$. We compare our algorithm with the well-known $\mbox{LinUCB}$\footnote{For LinUCB, we use the version decribed in \cite{lattimore2020bandit}, chapter 19, Theorem 19.2.}. The simulation result is shown in \figg{Lin_experiment}, $\lambda$ is set to $1$ throughout the numerical experiments. It can be observed that $\mbox{QLinUCB}$ has lower regret than $\mbox{LinUCB}$, and they both have small variance. 

  \begin{figure}[ht]
    \centering
      \centering
      \subfigure[Comparison between $\mbox{LinUCB}$ and $\mbox{QLinUCB}$]{
          \includegraphics[width=0.6\textwidth]{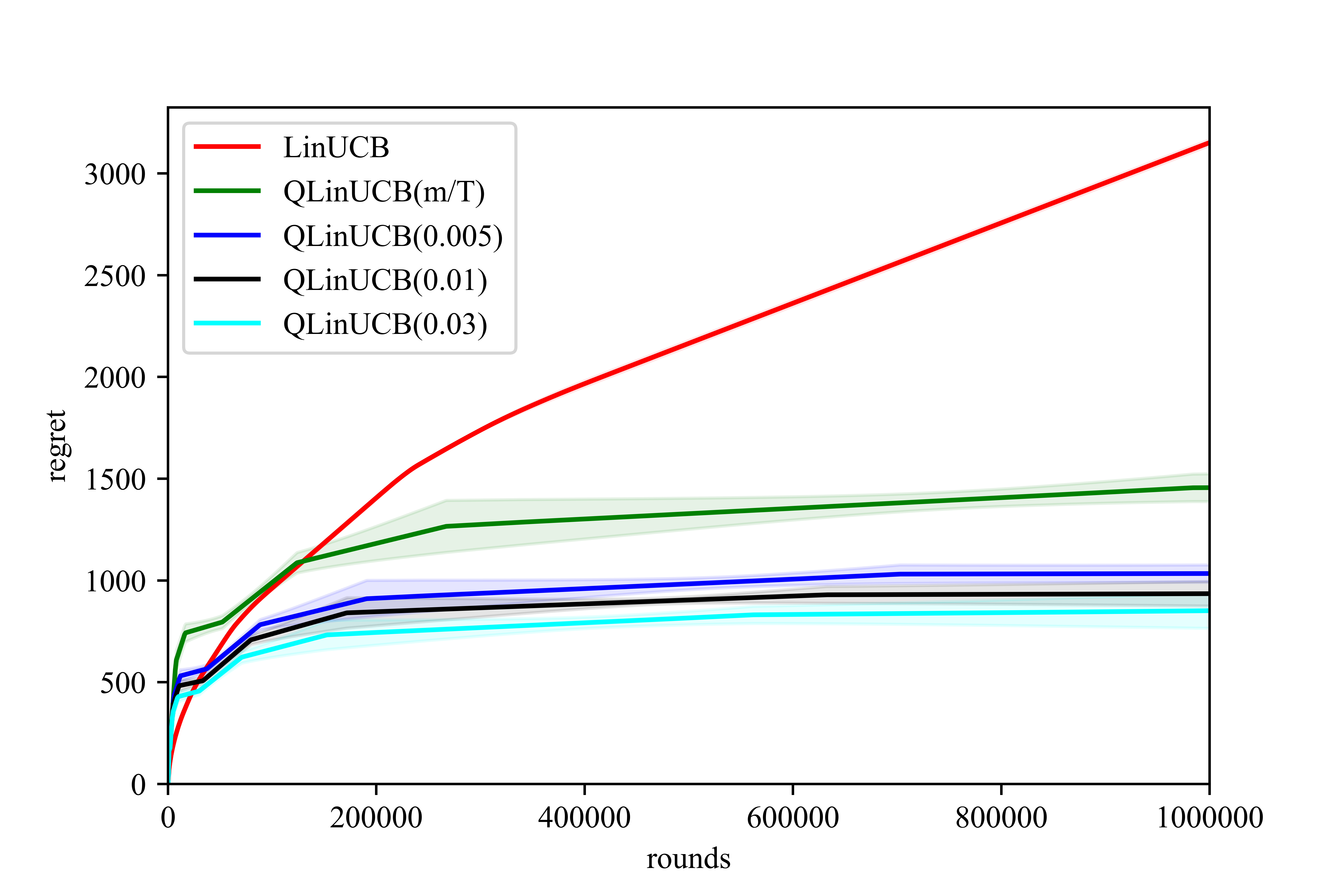}
          }
      \caption{Simulation for the QSLB setting with different $\delta$ (chosen in the legend).}
      \label{fig:Lin_experiment}
  \end{figure}
  \begin{figure}[ht]
      \centering
      \subfigure[QMAB]{
          \includegraphics[width=0.4\textwidth]{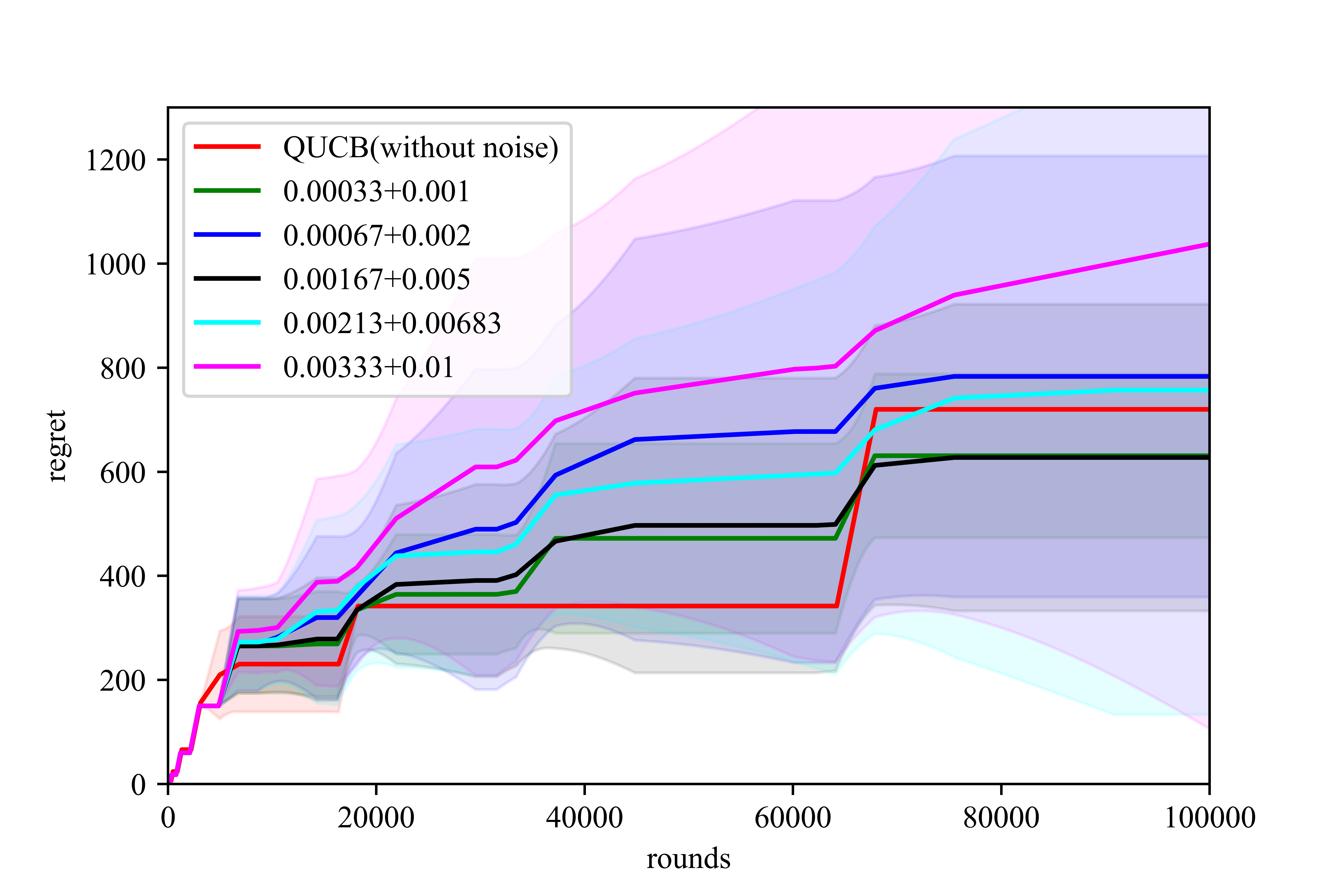}
      }
      \subfigure[QSLB]{
          \includegraphics[width=0.4\textwidth]{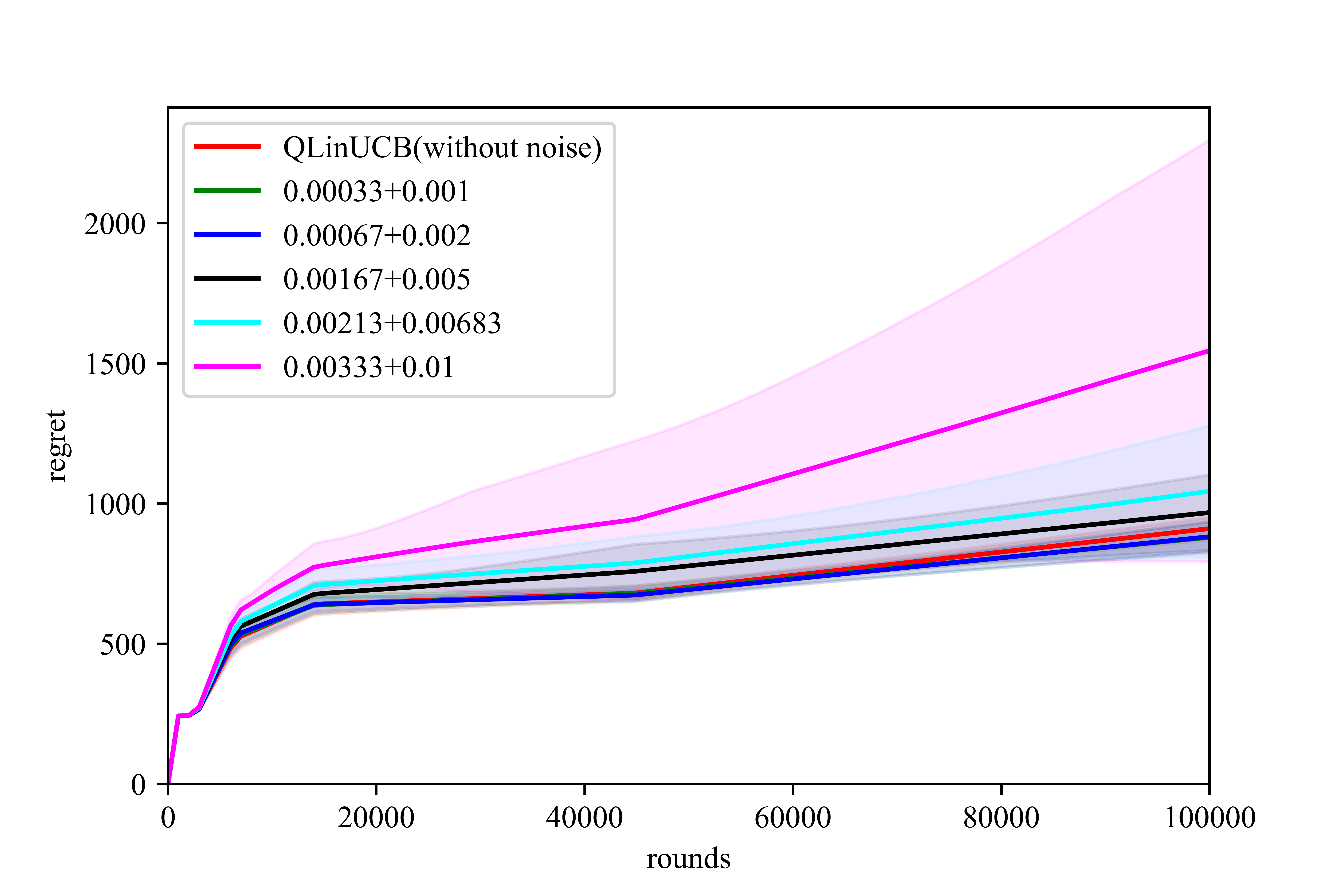}
      }
      \caption{The effect of the quantum depolarizing noise. The number $0.00333+0.01$ in the legend means we set \mbox{err} of a single-qubit channel to be $0.00333$ and $\mbox{err}$ of a two-qubit channel to be $0.01$. We set $\delta=1/T$ and $m/T$ for these two experiments, respectively.}
      \label{fig:noise}
  \end{figure}

\subsection{Experiments with depolarizing noise}
\label{sect:noise}
To study the effects of quantum noise on our two algorithms, we conduct simulations with depolarizing noise channel using the Python open-source quantum computing toolkit ``Qiskit"~\cite{Qiskit}.

\paragraph{Noise model.} We choose $\{\mbox{U1},\mbox{U2},\mbox{U3},\mbox{CNOT}\}$ as our basis gates set, and we consider the depolarizing noise model. This model is considered in quantum computers with many qubits and gates, including Sycamore~\cite{arute2019quantum}.
The depolarizing noise channel $E$ is defined as an operator acting on the density operator:
$E(\rho) = (1-\mbox{err})\rho+\mbox{err}\cdot\tr(\rho)\frac{I_m}{2^m}$, 
where $\mbox{err}$ is the error rate and $m$ is the number of qubits of this channel. 
For error rate, we try $\{0.001,0.002,0.005,0.01\}$ for two qubits channel, and set the single qubit error rate as the $\frac{1}{3}$ of the two qubit error rate. Besides, we also try single qubit error of $0.00213$ and two qubit error of $0.00683$, which are the error rate of Sycamore. In their paper, the error rate is $0.0016$ and $0.0063$, but these are pauli error rates. By the relation of pauli error and depolarizing error, we can convert them into our depolarizing error rate of $0.00213$ and $0.00683$.

\paragraph{Results.} For the QMAB setting, we choose the $2$-arm instance with rewards $\{0.4,0.5\}$. For the QSLB setting, we use the same instance as in \sect{no noise}. Since simulating quantum noise channel on a classical computer is time consuming, we set the time horizon of the bandit instances to be $10^5$.
Note that since the time horizon in our instance is relatively small, the quantum variants of UCB cannot outperform the classical UCB in this case even though the quantum variants have advantages in asymptotic order. This is because $\mbox{QMC}$ introduces large constant factors. In this section, we only focus on the performance of the quantum algorithms with different error rates. We plot the results in \figg{noise}.
From the figures, the expected regret of $\mbox{QUCB}$ is not affected much by depolarizing noise when the two-qubit depolarizing error rate is at most $0.00683$. Even if the error rates are $0.00333$ and $0.01$, the regret of $\mbox{QUCB}$ is still much better than pulling arms randomly (which incurs an expected regret of $5000$ at the end). However, the presence of depolarizing noise does increase the variance of the regret. As for $\mbox{QLinUCB}$, when the two-qubit channel error rate is no more than $0.00683$, the regret keeps almost unchanged and small variance all the time. When the two-qubits channel error rate is $0.01$, $\mbox{QLinUCB}$ suffers from higher expected regret and variance. Overall, at the level of noise that can be achieved by today's quantum computers, the regrets of our algorithms are only relatively little affected when time horizon is in the order of $10^5$.


\section{Conclusion}

In this paper, we proved that quantum versions of multi-arm bandits (\sect{mab}) and stochastic linear bandits (\sect{slb}) both enjoy $O(\poly(\log T))$ regrets. To the best of our knowledge, this is the first provable quantum speedup for regrets of bandit problems and in general exploitation in reinforcement learning. Compared to previous literature on quantum exploration algorithms for MAB and reinforcement learning, our quantum input model is simpler and only assumes quantum oracles for each individual arm. We also corroborate our results with numerical experiments (\sect{numerics}).

Our work raises several natural questions for future investigation. First, it is natural to seek for quantum speedups for regrets of other bandit problems. Second, additional research is needed to achieve speedup in $n$ and $d$ for the regrets of MAB and SLB, respectively; this may require a reasonable new model. Third, it is worth understanding whether algorithms with $T$-independent regret exist, or we can prove a matching quantum lower bound.


\section*{Acknowledgements}
We thank Ronald de Wolf for helpful discussions. TL was supported by a startup fund from Peking University, and the Advanced Institute of Information Technology, Peking University.


\bibliographystyle{plain}
\bibliography{Qbandits}


\newpage
\appendix

\section{QMAB with bounded variance assumption}\label{append:sec-3-proof}
$\mbox{QUCB}_2$ uses $\mbox{QMC}_2$ as the quantum mean estimator and modifies the update of $N_i$ the stage length in Line $3$ and $10$ of \algo{qucb2}, correspondingly. It also changes the initial accuracy to $2\sigma$ in Line 2 to fit the $\epsilon \leq 4\sigma$ requirement of $\mbox{QMC}_2$.
\begin{algorithm}[htbp]
	\caption{$\mbox{QUCB}_2(\delta)$}
	\label{algo:qucb2}
	\hspace*{\algorithmicindent} \textbf{Parameters:} fail probability $\delta$
	\begin{algorithmic}[1]
	   \FOR{$i=1\to n$}
	   	    \STATE $r_i\gets 2\sigma$
	        \STATE $N_i\gets \frac{C_2\sigma}{r_i}\log_2^{3/2}\left(\frac{8\sigma}{r_i}\right)\log_2 \log_2\left(\frac{8\sigma}{r_i}\right)\log \frac{1}{\delta}$
	        \STATE play arm $i$ for consecutive $N_i$ rounds
	        \STATE run $\mbox{QMC}_2(\mathcal{O}_i, r_i, \delta)$ to get an estimation $\hat{\mu}(i)$ for $\mu(i)$
	   \ENDFOR
		\FOR{each stage $s=1,2,\ldots$ (terminate when we have used $T$ queries to all $\mathcal{O}_{i}$)}
		\STATE Let $i_s\leftarrow\argmax_i \hat{\mu}(i)+r_i$ (if $\argmax$ has multiple choices, pick an arbitrary one) 
		\STATE update $r_{i_s} \gets r_{i_s}/2$
		\STATE update $N_{i_s} \gets \frac{C_2\sigma}{r_{i_s}}\log_2^{3/2}\left(\frac{8\sigma}{r_{i_s}}\right)\log_2\log_2 \left(\frac{8\sigma}{r_{i_s}}\right)\log \frac{1}{\delta}$
		\STATE Play $i_s$ for the next $N_{i_s}$ rounds, update $\hat{\mu}(i_s)$ by runing $\mbox{QMC}_2(\mathcal{O}_{i_s},r_{i_s}, \delta)$
		\ENDFOR
	\end{algorithmic}
\end{algorithm}
\begin{lemma}\label{lem:integrate by part}
For $K>0$ an integer, the following holds, 
\begin{align}
    \sum_{k=1}^K 2^k\log_2^{3/2}\left(2^{k+2}\right)\log_2\log_2 \left(2^{k+2}\right)\leq O\left(2^K K^{3/2}\log(K)\right).
\end{align}
\end{lemma}
\begin{proof}
We use the standard integrating by part technique:
\begin{align}
    &\sum_{k=1}^K 2^k\log_2^{3/2}\left(2^{k+2}\right)\log_2\log_2 \left(2^{k+2}\right)\\\leq & \int_{1}^{K+1}2^{x}\log_2^{3/2}\left(2^{x+2}\right)\log_2\log_2 \left(2^{x+2}\right)dx\\=&O\left(\int_{1}^{K+1}2^{x}(x+2)^{3/2}\log (x+2)dx\right)\\
    =&O\left(\int_1^{K+1}(x+2)^{3/2}\log (x+2)d2^x\right)\\=&O\left(2^{K+1}(K+3)^{3/2}\log(K+3)-\int_{1}^{K+1}2^x\frac{3}{2}(x+2)^{1/2}\left(1+\log (x+2)\right)dx\right)\\=&O\left(2^K K^{3/2}\log(K)\right).
\end{align}
The last equality holds because $\int_{1}^{K+1}2^x\frac{3}{2}(x+2)^{1/2}\left(1+\log (x+2)\right)dx> 0$.
\end{proof}

\begin{proof}[Proof of \thm{bv qucb regret}] The proof of \thm{bv qucb regret} is very similar to the proof of \thm{qucb regret}, and we use the same notation. The number of rounds of each stage in $\mathcal{S}_i$ are $C_2\log^{3/2}_2(8)\log_2\log_2 (8)\log(\frac{1}{\delta})$, $2C_2\log_2^{3/2}(16)\log_2\log_2 (16)\log(\frac{1}{\delta}),\ldots,2^{K_i-1}\log_2^{3/2}(4\cdot 2^{K_i})\log_2\log_2 (4\cdot 2^{K_i})\log(\frac{1}{\delta})$. 
It shows
\begin{align}
    n\frac{C_2}{2}\cdot 2^{3/2}\log(\frac{1}{\delta})+\sum_{i=1}^n\sum_{k=1}^{K_i}2^{k-1}\log_2^{3/2}\left(2^{k+2}\right)\log_2\log_2 \left(2^{k+2}\right)\log\left(\frac{1}{\delta}\right)\leq T.
\end{align}
Throwing away the $\log_2^{3/2}\left(2^{k+2}\right)\log_2\log_2 \left(2^{k+2}\right)$ term, we have
\begin{align}
    T&\geq n\frac{C_2}{2}\cdot 2^{3/2}\log(\frac{1}{\delta})+\Omega\left(\sum_{i=1}^n\sum_{k=1}^{K_i}2^{k-1}\log\left(\frac{1}{\delta}\right)\right)\\&=n\frac{C_2}{2}\cdot 2^{3/2}\log(\frac{1}{\delta})+\Omega\left(\sum_{i=1}^n (2^{K_i}-1)\log\left(\frac{1}{\delta}\right)\right),
\end{align}
which leads to an upper bound of number of stages
\begin{align}
    \sum_{i=1}^{n}K_i \leq O\left(n\log(T)\right).
\end{align}
By the same argument, we bound $R(T;i)$
\begin{align}
    R(T;i)&\leq O(2^{1-K_i}\sigma\sum_{k=0}^{K_i}2^k \log_2^{3/2}(2^{k+2})\log\log (2^{k+2}))\log(\frac{1}{\delta})\\
    &\leq O\left(\sigma K_i^{3/2}\log(K_i)\log\left(\frac{1}{\delta}\right)\right)\leq O\left(\sigma\log^{3/2} (T)\log\log (T)\log\left(\frac{1}{\delta}\right)\right),
\end{align}
where the second inequality is due to \lem{integrate by part}, and the last inequality is because for any $i$, $K_i\leq \log_2 (T)$. Therefore,
\begin{align}
    R(T)&=\sum_{i=1}^n R(T;i)\leq O\left(n\sigma\log^{3/2} (T)\log\log (T)\log\left(\frac{1}{\delta}\right)\right).
\end{align}
The upper bound of $\E [R(T)]$ is the same as \cor{qucb ex}.
\end{proof}

\section{QSLB with bounded variance assumption}\label{append:sec-4-proof}
For bounded variance rewards, the algorithm $\mbox{QLinUCB}_2$ uses $\mbox{QMC}_2$ to replace $\mbox{QMC}_1$, and it changes the length of stages in line $6$ of \algo{bv quantum linear bandits} to fit the query complexity of $\mbox{QMC}_2$. To satisfy the requirement of $\mbox{QMC}_2$ on the accuracy, we set the regularization parameter $\lambda > \frac{1}{4\sigma L}$. 

\begin{algorithm}[ht]
\caption{$\mbox{QLinUCB}_2(\delta)$}
\label{algo:bv quantum linear bandits}
	\hspace*{\algorithmicindent} \textbf{Parameters:} fail probability $\delta$
\begin{algorithmic}[1]
    \STATE Initialize $V_0\gets \lambda I_d$, $\hat{\theta}_0\gets \mathbf{0}\in \mathbb{R}^d$ and $m\gets d\log(\frac{L^2T^2}{d\lambda}+1)$.
	\FOR{each stage $s=1,2,\cdots$}
	    \STATE $\mathcal{C}_{s-1}\gets\{\theta\in\mathbb{R}^d: \|\theta-\hat{\theta}_{s-1}\|_{V_{s-1}}\leq \lambda^{1/2}S+\sqrt{d(s-1)}\}$.
	    \STATE $(x_s,\tilde{\theta}_s)\gets\argmax_{(x,\theta)\in A\times\mathcal{C}_{s-1}} x^{\trans}\theta$.
	    \STATE $\epsilon_{s} \gets \|x_s\|_{V_{s-1}^{-1}}$.
	    \FOR{the next $\frac{C_2\sigma}{\epsilon_s}\log_2^{3/2}(\frac{\sigma}{\epsilon_s})\log_2 \log_2 (\frac{\sigma}{\epsilon_s})\log\left(\frac{m}{\delta}\right)$ rounds}
	    \STATE Play action $x_s$ and run $\mbox{QMC}_2(\mathcal{O}_{x_s},\epsilon_s,\delta/m)$, getting $y_{s}$ as an estimation of $x_s^{\trans}\theta^*$.
	    \ENDFOR
	    \STATE $V_s\gets V_{s-1}+\frac{1}{\epsilon_{s}^2}x_sx_s^{\trans}$.
	    \STATE $X_s\gets (x_1,x_2,\ldots, x_s)^{\trans}\in\mathbb{R}^{s\times d}$, $Y_s\gets(y_1,y_2,\ldots,y_s)^{\trans}\in \mathbb{R}^{s}$ and  $W_s \gets \diag\left(\frac{1}{\epsilon_1^2},\frac{1}{\epsilon_2^2},\ldots, \frac{1}{\epsilon_s^2}\right)$
	    \STATE $\hat{\theta}_{s}\gets V_s^{-1}X_{s}^{\trans}W_{s}Y_{s}$.
	\ENDFOR
\end{algorithmic}
\end{algorithm}
\begin{proof}[Proof of \thm{bv qlb regret}]
    It is easy to check the modification does not influence the results of \lem{stages} and \lem{trust region}, so we have the \algo{bv quantum linear bandits} version of both lemmas. Then follow the proof of \thm{qlb regret} until we get
    \begin{align}
        (x^{*}-x_s)^{\trans}\theta^* \leq 2\epsilon_s\cdot (\lambda^{1/2}S+\sqrt{d(s-1)}).
    \end{align}
    due to our modification on the length of stages, the cumulative regret in stage $s$ is bounded by 
    \begin{align}
        2C_2\sigma(\lambda^{1/2}S+\sqrt{d(s-1)})\log_2^{3/2}\left(\frac{8\sigma}{\epsilon_s}\right)\log_2\log_2\left(\frac{8\sigma}{\epsilon_s}\right)\log \left(\frac{m}{\delta}\right)
    \end{align}
    Since $\frac{1}{\epsilon_s}\leq T$, the cumulative regret in stage $s$ is bounded by 
    \begin{align}
        O\left(\sigma\left(\lambda^{1/2}S+\sqrt{d(s-1)}\right)\log_2^{3/2}\left(T\sigma\right)\log_2\log_2\left(T\sigma\right)\log \left(\frac{m}{\delta}\right)\right)
    \end{align}
    Then,
    \begin{align}
    R(T) & \leq \sum_{s=1}^{m} O\left(\sigma\sqrt{d(s-1)}\log_2^{3/2}\left(T\sigma\right)\log_2\log_2\left(T\sigma\right)\log \left(\frac{m}{\delta}\right)\right) \\
        &= O\left(\sigma d^2\log^3 (\sigma LT)\log\log (\sigma LT)\log \left(\frac{m}{\delta}\right)\right).
    \end{align}
    The upper bound of $\E [R(T)]$ follows accordingly.
\end{proof}

\section{More Details about Numerical Experiments}\label{append:more-numerics}
For Bernoulli rewards, we use Quantum Amplitude Estimation as our quantum mean estimator.
\begin{lemma}[Quantum Amplitude Estimation~\cite{brassard2002amplitude}]
	\label{lem:AE_new}
	Assume that $y$ is a Bernoulli random variable and $\mathcal{O}$ encodes $y$ and $d$ an positive integer. Then $\mbox{QAE}(\mathcal{O},d)$ returns an estimate $\hat{y}$ of $\mathbb{E}[y]$ such that \[|\hat{y}-\E[y]|\leq \frac{\pi}{2^d}+\frac{\pi^2}{2^{2d}}\] with probability at least $\frac{8}{\pi^2}$. It use exactly $2^d$ quantum queries to $\mathcal{O}$ and $\mathcal{O}^{\dagger}$.
\end{lemma}

Through \lem{AE_new}, for any given accuracy $\epsilon$, we can find the smallest positive integer $d_{\epsilon}$ such that $\mbox{QAE}(\mathcal{O},d_{\epsilon})$ return an estimate $\hat{y}$ such that $|\hat{y}-\mathbb{E}[Y]|\leq \epsilon$ with probability at least $\frac{8}{\pi^2}$. To apply $\mbox{QAE}$ on $\mbox{QUCB}$ and $\mbox{QLinUCB}$, this success probability needs to be improved to any given $1-\delta$. This can be done by simply running $\mbox{QAE}(\mathcal{O},d_{\epsilon})$ independently for several times and selecting the median of these estimates as the final estimate. To specify the constant in this process, we give the following lemma.

\begin{lemma}[Powering Lemma~\cite{jerrum1986random}]
Run $\mbox{QAE}(\mathcal{O},d_{\epsilon})$ independently for $5\log (1/\delta)$ times, then the median $\tilde{y}$ of these estimations satisfies $|\tilde{y}-\mathbb{E}[y]|\leq \epsilon$ with probability at least $1-\delta$.
\end{lemma}
\begin{proof}
    If $|\tilde{y}-\mathbb{E}[y]|>\epsilon$, then at least half of the estimations do not lie in $\left[\mathbb{E}[y]-\epsilon,\mathbb{E}[y]+\epsilon\right]$. Each estimation do not in $\left[\mathbb{E}[y]-\epsilon,\mathbb{E}[y]+\epsilon\right]$ happens with a probability less than $1-\frac{8}{\pi^2}$. Let $\{X_i\}_{i=1}^{5\log (1/\delta)}$ be i.i.d samples draw from Bernoulli distribution with mean value $\frac{8}{\pi^2}$. Then
    \begin{align}
        \mathbb{P}\left[|\tilde{y}-\mathbb{E}[y]|>\epsilon\right]&\leq\mathbb{P}\left[\frac{1}{5\log (1/\delta)}\sum_{i=1}^{5\log (1/\delta)} X_i\leq \frac{8}{\pi^2}-\left(\frac{8}{\pi^2}-\frac{1}{2}\right)\right]\\&\leq \left(\left(\frac{8/\pi^2}{1/2}\right)\left(\frac{1-8/\pi^2}{1/2}\right)\right)^{\frac{5}{2}\log(1/\delta)}\leq\delta,
    \end{align}
    where the second inequality is due to Chernoff's bound.
\end{proof}

\paragraph{Simulating quantum subroutine.}  We noticed that Theorem 11 of \cite{brassard2002amplitude} give the closed-form description of the measure results distribution of the Quantum Amplitude Estimation circuit. Thus we can directly use this closed-form description to construction a classical distribution and sample from it to simulate the result of Quantum Amplitude Estimation.

\paragraph{Solver for $(x_s,\tilde{\theta}_s)$.} Generally speaking, $\argmax_{(x,\theta)\in A\times\mathcal{C}_{s-1}} x^{\trans}\theta$ is a bilinear optimization problem, which has no efficient solver. But when we consider the finite action set, there is a simple method to calculate $x_s$. That is, if $\mathcal{C}_{s-1}$ is a ellipsoid defined by $\{\theta \ |\ \|\theta-\hat{\theta}_{s-1}\|_{V_{s-1}}\leq r_{s-1}\}$, then
\[x_s =\argmax_{x\in A}\  \langle \hat{\theta}_{s-1},x\rangle + r_{s-1}\|x\|_{V_{s-1}^{-1}}\]
This can be solved by enumerating all possible $x\in A$ if $A$ is a small finite set. Now we consider the choice of $r_{s-1}$, we can let $r_{s-1}=\lambda^{1/2}S+\sqrt{d(s-1)}$ directly from \lem{trust region}. But \lem{trust region} actually shows a smaller confidence region as shown in~\eq{ds bound}. $\|W_{s}^{1/2}X_{s}V_{s}^{-1}X_{s}^{\trans}W_{s}^{1/2}\|_2$ in the right side of~\eq{ds bound} can be calculated exactly in stage $s+1$. Therefore we set $r_{s-1}=\lambda^{1/2}S+\sqrt{(s-1)\|W_{s-1}^{1/2}X_{s-1}V_{s-1}^{-1}X_{s-1}^{\trans}W_{s-1}^{1/2}\|_2}$ in our implementation.
\end{document}